\def\year{2020}\relax
\documentclass[letterpaper]{article} 
\usepackage{aaai20}  
\usepackage{times}  
\usepackage{helvet} 
\usepackage{courier}  
\usepackage[hyphens]{url}  
\usepackage{graphicx} 
\usepackage{graphicx}  
\usepackage{amssymb}
\usepackage{amsmath}
\usepackage{amsthm}
\usepackage[vlined, ruled]{algorithm2e}
\usepackage{graphicx}
\usepackage{booktabs}

\theoremstyle{plain}
\newtheorem{thm}{Theorem}

\newtheorem{prop}[thm]{Proposition}

\theoremstyle{definition}
\newtheorem{defn}{Definition}

\title{Projective Quadratic Regression for Online Learning}
\author{Wenye Ma\\
Tencent\\ 
wenyema@tencent.com 
}
\begin{document}

\maketitle

\begin{abstract}
This paper considers online convex optimization (OCO) problems - the paramount framework for online learning algorithm design. The loss function of learning task in OCO setting is based on streaming data so that OCO is a powerful tool to model large scale applications such as online recommender systems. Meanwhile, real-world data are usually of extreme high-dimensional due to modern feature engineering techniques so that the quadratic regression is impractical. Factorization Machine as well as its variants are efficient models for capturing feature interactions with low-rank matrix model but they can't fulfill the OCO setting due to their non-convexity. In this paper, We propose a projective quadratic regression (PQR) model. First, it can capture the import second-order feature information. Second, it is a convex model, so the requirements of OCO are fulfilled and the global optimal solution can be achieved. Moreover, existing modern online optimization methods such as Online Gradient Descent (OGD) or Follow-The-Regularized-Leader (FTRL) can be applied directly. In addition, by choosing a proper hyper-parameter, we show that it has the same order of space and time complexity as the linear model and thus can handle high-dimensional data. Experimental results demonstrate the performance of the proposed PQR model in terms of accuracy and efficiency by comparing with the state-of-the-art methods.
\end{abstract}

\begin{figure*}
\includegraphics[width=1\textwidth]{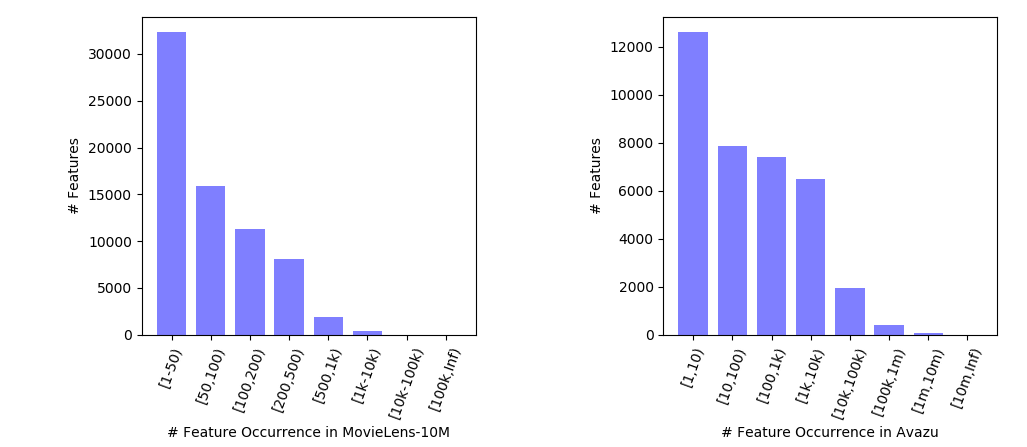}
\caption{Distribution of Feature Occurrence}
\label{fig.feature_counts}
\end{figure*}

\section{Introduction}
In the setting of online learning, the training data arrive in a streaming fashion and the system should provide a response immediately. Specifically, online learning is performed in a sequence of consecutive rounds indexed by time $t$. At round $t$, the learner updates a learning variable using simple calculations that are usually based on a single sample only. Online learning is a powerful tool for a wide range of applications such as online rating, news feeding, and ad-click prediction. The paramount framework for online learning is online convex optimization (OCO) \cite{HazanOCO,ShalevOCO} which can be solved by many state-of-the-art online optimization algorithms such as Online Gradient Decent (OGD) \cite{ZinkevichICML} or Follow-The-Regularized Leader (FTRL) \cite{McMahanKDD}. It is worth to note that the study of OCO and theoretical analysis about above algorithms are based on convex requirement of OCO. The convexity is two folds. First, the loss function is convex. Second, the feasible set of learning variables is also convex.

Meanwhile, in real-world applications, the feature dimension could be extremely high. Therefore, traditional linear models are still the most popular model in such scenario. However, linear models fail to utilize the interaction between features so that lots of artificial feature engineering work is needed. Quadratic regression captures first-order information of each input feature as well as second-order pairwise feature interactions. However, it is usually unacceptable in terms of space complexity when dealing with high-dimensional data because the model size is of order $O(d^2)$ with $d$-dim features. Factorization Machine (FM) \cite{RendleICDM} is an efficient mechanism for capturing up to second-order feature information with a low-rank matrix in factorized form. It only needs $O(md)$ space with rank $m$. FM achieves state-of-the-art performance in various applications \cite{JuanWWW,LiDifacto,LuWSDM,NguyenSIGIR,RendleSIGIR,ZhongCIKM,ChenFMVAE,ChenRaFM} and has recently regained significant attention from researchers \cite{BlondelNIPS,BlondelICML,HeIJCAI,XuICDM}. Unfortunately, the vanilla FM \cite{RendleICDM} and its variants \cite{BlondelNIPS,ChengACM,JuanWWW,LuWSDM} are non-convex formulations and unable to fulfill the fundamental requirements demanded in OCO. Several convexified FM formulations \cite{BlondelECML,LinWWW,YamadaKDD} are proposed to address this weakness but the computation cost is high and thus impractical in high-dimensional problems.

The above formulations are all based on the low-rank assumption of the feature-interaction matrix and thus to reduce the number of learning variables. However, the question is whether we really need the low-rank assumption? Is there any other way to save parameters? Our answer is that we don't need the low-rank assumption and we can save parameters by the intuition of sharing parameters.

This intuition is based on our observation that the frequencies of different features occurring in real-world datasets vary a lot. For instance, figure~\ref{fig.feature_counts} shows the numbers of occurrences of different features in two public datasets, i.e. MovieLens\footnote{\url{https://grouplens.org/datasets/movielens/}} and Avazu\footnote{\url{https://www.kaggle.com/c/avazu-ctr-prediction/data}} respectively. It is shown that a large number of features scarcely occur in the datasets, while only a few features frequently occur. So we separate the features into two categories: high frequent features and low frequent features. We assume that the high frequent features are more important so that their interaction are more valuable. To this background, we propose a Projective Quadratic Regression (PQR) model. Specifically, for high frequent features, we adopt the full quadratic form. Within low frequent features, we just ignore their interactions. We then use a set of sharing parameters for high and low frequent feature interactions. In fact, we can rewrite the global bias, the feature weight vector and the feature interaction matrix into an augmented symmetric matrix. The intuiton of sharing parameters we mentioned above is equivalent to restricting the feasible set to be a subset of symmetric matrices. We show that the feasible set are actually convex and then satisfy the first requirement of OCO. Then we rewrite the loss function into a convex function with respect to the augmented matrix and thus fulfill the second requirement of OCO.  Based on this scheme, the resulting formulation of PQR can seamlessly meet the aforementioned requirements of the OCO framework. Moreover, due to its convexity, most optimization algorithms such as FTRL can be directly applied to PQR model and theoretical analyses such as regret bounds and convergence rates are still valid. In addition, by choosing a proper hyper-parameter, this kind of matrices can be represented by $O(d)$ parameters so that it has the same space complexity with linear models.

We conduct extensive experiments on real-world datasets to evaluate the empirical performance of PQR. As shown in the experimental results in both online rating prediction and online binary classification tasks, PQR outperforms state-of-the-art online learning algorithms in terms of both accuracy and efficiency, especially for high-dimensional problems.

\paragraph{Contribution}

\begin{itemize}
    \item PQR model captures first-order feature information as well as second-order pairwise feature interaction. It is a convex formulation so that it fulfills the requirements of online convex optimization setting and the optimal solution is theoretically guaranteed.
    \item Optimization algorithms for convex models can be applied directly to PQR and all the theoretical analyses are still valid.
    \item PQR is a general framework and can be applied to many other tasks such as batch and stochastic settings besides online learning.
    \item PQR is efficient in terms of space complexity and computation cost. It has the same order of space and time complexity with linear model with a proper hyper-parameter.
    \item We evaluate the performance of the proposed PQR model on both online rating prediction and online binary classification tasks. Experimental results show that PQR outperforms state-of-the-art online learning methods.
\end{itemize}

\section{Method}\label{pqr_sec}
In this section, we first introduce preliminaries in online learning and then turn to the development of the proposed projective quadratic regression (PQR) model for online convex problem. The PQR model is essentially a quadratic regression model with a constraint on the feasible set. The idea behind this model is sharing parameters. We prove the convexity of PQR so that it fits into the online learning setting and state-of-the-art online optimization algorithms can be applied directly. 
\subsection{Online Convex Optimization}\label{oco_subsec}

The online learning algorithms are built upon the assumption that the training instances arrive sequentially rather than being available prior to the learning task. The paramount framework for online learning is Online Convex Optimization (OCO) \cite{HazanOCO,ShalevOCO}. It can be seen as a structured repeated game between a learner and an adversary. At each round $t \in \{1, 2, \ldots, T\}$, the learner is required to generate a decision point $\mathbf{\theta}_t$ from a convex set $\mathcal{S} \subseteq \mathbb{R}^n$. Then the adversary replies to the learner’s decision with a convex loss function $f_t : \mathcal{S} \longrightarrow \mathbb{R}$ and the learner suffers the loss $f_t(\mathbf{\theta}_t)$. Specifically, the online learning setting can be described as follows.

For index $t = 1, 2, \ldots, T$, the learner chooses learning parameters $\mathbf{\theta}_t \in \mathcal{S}$, where $\mathcal{S}$ is convex. Then the environment responds with a convex function $f_t: \mathcal{S} \longrightarrow \mathbb{R}$ and the output of the round is $f_t(\mathbf{\theta}_t)$.

The regret is defined by
\begin{align}
\mbox{regret}_T = \sum^T_{t=1}f_t(\mathbf{\theta}_t) - \min\limits_{\mathbf{\theta}^\ast\in \mathcal{S}}\sum^T_{t=1}f_t(\mathbf{\theta}^\ast).   
\end{align}

The goal of the learner is to generate a sequence of decisions $\{\mathbf{\theta}_t | t = 1, 2, \ldots, T\}$ so that the regret with respect to the best fixed decision in hindsight is sub-linear in $T$, i.e. $\mbox{lim}_{T\rightarrow\infty}\mbox{regret}_T/T = 0$. The sub-linearity implies that when $T$ is large enough, the learner can perform as well as the best fixed decision in hindsight.

Based on the OCO framework, many online learning algorithms have been proposed and successfully applied in various applications \cite{DeMarzoSTOC,LiSurveys,ZhaoKDD}. Follow-The-Regularized-Leader (FTRL) \cite{McMahanKDD} is one of the most popular online learning algorithms, which is summarized in Algorithm \ref{ftrl_alg} and the regret bound of FTRL is shown in Theorem \ref{regret_bound_ftrl} \cite{HazanOCO}.

\begin{algorithm}
\SetAlgoLined
 Input: $\eta>0$, regularization function $R$, and convex set $\mathcal{S}$
 
 Initialize: $\mathbf{\theta}_1 = \arg\min\limits_{\mathbf{\theta}\in \mathcal{S}} R(\mathbf{\theta})$
 
 \For{$t=1$ to $T$}{
  Predict $\mathbf{\theta}_t$
  
  Observe the loss function $f_t$ and let $\mathbf{g}_t = \nabla f_t(\mathbf{\theta}_t)$
  
  Update
  
  $\mathbf{\theta}_{t+1} = \arg\min\limits_{\mathbf{\theta}\in \mathcal{S}} \{\eta\sum^t_{s=1}\mathbf{g}_s \cdot \mathbf{\theta} + R(\mathbf{\theta})\}$
  
 }
 \caption{Follow-The-Regularized-Leader}\label{ftrl_alg}
\end{algorithm}

\begin{thm}\label{regret_bound_ftrl}
The FTRL Algorithm \ref{ftrl_alg} attains the following bound on regret
$$\mbox{regret}_T\leq 2\eta\sum^T_{t=1} ||\mathbf{g}_t||_F^2 + \frac{R(\mathbf{\theta}) - R(\mathbf{\theta}_1)}{\eta}.$$
If an upper bound on the local norms is known, i.e. $||\mathbf{g}_t||_F \leq G_R$ for all time $t$, then we can further optimize over the choice of $\eta$ to obtain
$$\mbox{regret}_T \leq 2 D_RG_R\sqrt{2T}.$$
\end{thm}

\subsection{Quadratic Regression}\label{qr_subsec}

The general quadratic regression is a combination of linear model with pairwise feature interaction. Given an input feature vector $\mathbf{x}\in\mathbb{R}^d$, the prediction $\hat{y}$ can be obtained with the following formula:
\begin{align}
\hat{y} = b + \mathbf{w}\cdot \mathbf{x} + \sum_{1\leq i \leq j \leq d} A_{i,j}\mathbf{x}_i\mathbf{x}_j
\end{align}
where $b\in\mathbb{R}$ is the bias term, $\mathbf{w}$ is the first-order feature weight vector and $A\in\mathbb{R}^{d\times d}$ is the second-order feature interaction matrix. Clearly we only need the upper triangle of the matrix so that we can assume the matrix is symmetric, i.e., $A=A^T\in\mathbb{R}^{d\times d}$. Then we rewrite the bias term $b$, the first-order feature weight vector $\mathbf{w}$ and the second-order feature interaction matrix $A$ into a single matrix $C$:
\begin{align}
C = \begin{bmatrix}
    A & \mathbf{w} \\
    \mathbf{w}^T & b \\
\end{bmatrix}    
\end{align}
Therefore, the prediction $\hat{y}$ can be represented by a quadratic form with extended feature vector $\hat{\mathbf{x}} = (\mathbf{x}^T, 1)^T$. In fact, we can define the quadratic regression model by
\begin{align} 
\hat{y}(C) &= \frac{1}{2}\hat{\mathbf{x}}^TC\hat{\mathbf{x}}\label{eq_qr_form} \\
 & C\in \mathcal{S} \label{eq_symmetric_constraint}
\end{align}
where $\mathcal{S}$ is the set of all $(d+1)\times (d+1)$ symmetric matrices. Clearly, $\hat{y}$ is linear and thus convex in $C$. Under this setting, we can show that the composition of any convex function and the quadratic regression form is still convex in $C$. Actually we have

\begin{prop}\label{convex_obj}
Let $\mathcal{K}\subseteq\mathbb{R}^{(d+1)\times (d+1)}$ be a convex set, $f: \mathbb{R}\longrightarrow\mathbb{R}$ a convex function, and $\hat{y}$ a quadratic form defined in (\ref{eq_qr_form}). Then $f \circ \hat{y}: \mathcal{K}\longrightarrow\mathbb{R}$ is convex in $C\in\mathcal{K}$.
\end{prop}
\begin{proof}
Consider $\lambda\in[0, 1]$ and $C_1, C_2 \in \mathcal{K}$. By the convexity of $\mathcal{K}$, $C = \lambda C_1 + (1 - \lambda) C_2 \in \mathcal{K}$. Then by the linearity of $\hat{y}$ and convexity of $f$, we have
\begin{align*} 
f(\hat{y}(C)) & = f(\hat{y}(\lambda C_1 + (1-\lambda)C_2)) \\
 &=  f(\lambda \hat{y}(C_1) + (1-\lambda)\hat{y}(C_2)) \\ 
 &\leq  \lambda f(\hat{y}(C_1)) + (1-\lambda)f(\hat{y}(C_2)),
\end{align*}
which shows $f\circ\hat{y}$ is convex.
\end{proof}

\begin{figure*}
    \centering
    \[\left[\begin{array}{llllllllll}
      0         & p_{1,2}   & p_{1,3}   & \ldots & p_{1,k-1} & p_{1,k}   & q_1     & q_1     & \ldots & q_1     \\
      p_{1,2}   & 0         & p_{2,3}   & \ldots & p_{2,k-1} & p_{2,k}   & q_2     & q_2     & \ldots & q_2     \\
      p_{1,3}   & p_{2,3}   & 0         & \ldots & p_{3,k-1} & p_{3,k}   & q_3     & q_3     & \ldots & q_3     \\
      \vdots    & \vdots    & \vdots    & \ldots & \vdots    & \vdots    & \vdots  & \vdots  &        & \vdots  \\
      p_{1,k-1} & p_{2,k-1} & p_{3,k-1} & \ldots & 0         & p_{k-1,k} & q_{k-1} & q_{k-1} & \ldots & q_{k-1} \\
      p_{1,k}   & p_{2,k}   & p_{3,k}   & \ldots & p_{k-1,k} & 0         & q_k     & q_k     & \ldots & q_k     \\
      q_1       & q_2       & q_3       & \ldots & q_{k-1}   & q_k       & 0       & 0       & \ldots & 0       \\
      q_1       & q_2       & q_3       & \ldots & q_{k-1}   & q_k       & 0       & 0       & \ldots & 0       \\
      \vdots    & \vdots    & \vdots    &        & \vdots    & \vdots    & \vdots  & \vdots  &        & \vdots  \\
      q_1       & q_2       & q_3       & \ldots & q_{k-1}   & q_k       & 0       & 0       & \ldots & 0
    \end{array}\right]\]
  \caption{$A^{(\mathcal{H}, \mathcal{L})}$ with $\mathcal{H}=\{1, 2, \ldots, k\}$ and $\mathcal{L}=\{k+1, k+2, \ldots, d\}$}
  \label{matrix_example}
\end{figure*}

\subsection{Projective Quadratic Regression}\label{pqr_subsec}

The major problem of quadratic regression is that the feasible set contains all symmetric matrices which is too large and impractical in real-world applications. Existing formulations are usually based on the low-rank assumption. The vanilla FM is just decomposing the matrix into a product of a low-dimensional matrix and its transpose. But this formulation makes it a non-convex model. Several convexified FM formulations \cite{LinWWW,YamadaKDD} are based on the same assumption and adding low-rank constraint by restricting the nuclear norm of the candidate matrix. However, it introduces operations such as singular value decomposition and the overall computation cost is actually high and thus it is impractical in real applications with large and high-dimensional datasets.

The main idea of our proposed projective quadratic regression model is to restrict the feasible set to a subset of symmetric matrices. It is not based on low-rank assumption but the intuition of sharing parameters. This is based on the assumption that the frequencies of different features occurring vary a lot as shown in Figure~\ref{fig.feature_counts}. We then separate the features into two categories: high frequent features and low frequent features. We assume that the high frequent features are more important so that their interaction are more valuable. Therefore, for high frequent feature interactions, we adopt the full quadratic form. And within low frequent features, we just ignore their interactions. Finally, we use a set of sharing parameters for high and low frequent feature interactions. 

We denote $\mathcal{I} = \{1, 2, \ldots, d\}$ as the index set of all features where $d$ is the dimension of the feature space. The separation of high and low frequent features can be defined as follows.

\begin{defn}\label{separation}
A feature separation is a bi-partition of index set $\mathcal{I}$, i.e.,
\begin{align}
    \mathcal{I} = \mathcal{H} \cup \mathcal{L} \mbox{ with } \mathcal{H}\cap\mathcal{L}=\emptyset,
\end{align}
where $\mathcal{H}$ is the set of indices of all high frequent features while $\mathcal{L}$ for low frequent features.
\end{defn}

In practice, we have various ways to separate the high and low frequent features. For example, we can sample the training data and calculates the features counts. Then we can choose the top-$k$ features as the high frequent features and the rest as low frequent features. Another way may be setting a threshold, any feature whose occurrence is greater than the threshold can be considered as a high frequent feature, and otherwise a low frequent feature. In this paper, we use the top-$k$ method where $k$ is a hyper-parameter and can be called the order of the model.

After separating the indices, i.e., $\mathcal{I} = \mathcal{H}\cup\mathcal{L}$, the key idea of PQR is using the following matrix to catch feature interactions:
\begin{align}\label{pqr_matrix}
\nonumber    A_{i,i} & = 0 \mbox{ for all } i \in\mathcal{I} \\
\nonumber    A_{i,j} & = p_{i,j} \mbox{ if } i<j \mbox{ and } i, j \in \mathcal{H} \\
\nonumber    A_{i,j} & = q_{i} \mbox{ if } i<j, i \in \mathcal{H} \mbox{ and } j \in \mathcal{L} \\ 
             A_{i,j} &= 0 \mbox{ if } i\neq j \mbox{ and } i, j \in \mathcal{L}
\end{align}
where $p_{i,j}$ and $q_i$ are learning parameters.

\begin{defn}\label{pqr_matrix_defn}
Given a separation $\mathcal{I}=\mathcal{H}\cup\mathcal{L}$, the matrix with the form in (\ref{pqr_matrix}) can be called a PQR matrix with respect to this separation. The PQR matrix can be denoted as $A^{(\mathcal{H}, \mathcal{L})}$.
\end{defn}

Let us assume that the size of $\mathcal{H}$ is $k$, i.e., $|\mathcal{H}| = k$. So $|\mathcal{L}| = d - k$. Without losing generality, if we assume $\mathcal{H} = \{1, 2, \ldots, k\}$ and $\mathcal{L} = \{k+1, k+2, \ldots, d\}$, then the matrix looks like in Figure~\ref{matrix_example}. We can easily see that there are $\frac{1}{2}k (k+1)$ learning parameters for the PQR matrix. Combining with the linear part of the model, the space complexity of learning parameters is of the order $O(k^2 + d)$. In particular, if we choose $k$ in the order of $O(\sqrt{d})$, the PQR model have the same order of computation cost as linear models. The computation cost is summarized in Table \ref{computation_complexity}.

\begin{table}
  \caption{Computation complexity}
  \label{computation_complexity}
  \centering
  \begin{tabular}{llll}
    \toprule
         & Linear Models & FM Models & PQR \\
    \midrule
    Time & $O(nd)$ & $O(nmd)$ & $O(nd + nk^2)$ \\
    Space & $O(d)$ & $O(md)$ & $O(d + k^2)$ \\
    \bottomrule
  \end{tabular}
\end{table}

Now we can show that the PQR model fulfills the OCO framework. For a given separation $\mathcal{I} = \mathcal{H}\cup\mathcal{L}$, Let

\begin{align}
\mathcal{T}^{(\mathcal{H}, \mathcal{L})} = \{ C | C = \begin{bmatrix}
    A^{(\mathcal{H}, \mathcal{L})} & \mathbf{w} \\
    \mathbf{w}^T & b \\
\end{bmatrix}\}
\end{align}
where $A^{(\mathcal{H}, \mathcal{L})}$ is the PQR matrix. We then have
\begin{prop}\label{convex_feasible}
$\mathcal{T}^{(\mathcal{H}, \mathcal{L})}$ is a convex set for any separation $\mathcal{I} = \mathcal{H}\cup\mathcal{L}$.
\end{prop}
\begin{proof}
It is easy to verify that $\mathcal{T}^{(\mathcal{H}, \mathcal{L})}$ is closed under addition and scalar multiplication. By the definition of convex set, $\mathcal{T}^{(\mathcal{H}, \mathcal{L})}$ is convex.
\end{proof}

Now we define our projective quadratic regression model (PQR) as a quadratic regression model with restriction the feature-interaction matrices to be the PRQ matrices. In summary, the proposed PQR model can be represented by
\begin{align} 
\hat{y}(C) &= \frac{1}{2}\hat{\mathbf{x}}^TC\hat{\mathbf{x}}\label{eq_pqr_form} \\
\nonumber & \mbox{s.t. } C\in \mathcal{T}^{(\mathcal{H}, \mathcal{L})} 
\end{align}

By Proposition \ref{convex_obj} and \ref{convex_feasible}, we can see that PQR model (\ref{eq_pqr_form}) fulfills the OCO setting.

\paragraph{Examples}
There are two important examples: online rating prediction and online binary classification.

In the online rating prediction task, at each round, the model receives a pair of user and item sequentially and then predicts the value of the incoming rating correspondingly. Denote the instance arriving at round $t$ as $(u_t, i_t, y_t)$, where $u_t$, $i_t$ and $y_t$ represent the user, item and the rating (a number from 1 to 5) given by user $u_t$ to item $i_t$. The user feature and item feature (still denoted by $u_t$ and $i_t$ to save notation) can be represented by a concatenation of one-hot vectors. The input feature vector $x_t$ is constructed as a concatenation of $u_t$ and $i_t$. Upon the arrival of each instance, the model predicts the rating with $\hat{y}_t = \frac{1}{2} \hat{x}_t C \hat{x}_t$, where $\hat{x}_t = [x^T_t , 1]^T$. The convex loss function incurred at each round is the squared loss:
\begin{align}
f_t(\hat{y}(C_t)) = \frac{1}{2}||\hat{y}(C_t) - y_t||^2_2    
\end{align}

The online binary classification task is usually applied to the click-through-rate (CTR) prediction problem. In this task, the instances are denoted as $(x_t , y_t)$ indexed by $t \in [1, 2, \ldots, T]$, where $x_t$ is the input feature vector and $y_t \in \{-1, 1\}$ is the class label. At round $t$, the model predicts the label with $\mbox{sign}(\hat{y}_t) = \mbox{sign}(\frac{1}{2}\hat{x}^T_t C_t\hat{x}_t)$, where $\hat{x}_t = [x^T_t , 1]^T$. The loss function is a logistic loss function:
\begin{align}
f_t(\hat{y}(C_t)) = \log(1 + \frac{1}{e^{-y_t\cdot\hat{y}(C_t)}})    
\end{align}

\begin{algorithm}
\SetAlgoLined
 Pre-calculate separation $\mathcal{I}=\mathcal{H}\cup\mathcal{L}$.
 
 Parameter: $\alpha>0$, $\beta>0$, $\lambda_1>0$, $\lambda_2>0$
 
 Initialize: $\mathbf{z} = \mathbf{n} = 0$
 
 \For{$t=1$ to $T$}{
  Receive feature vector $\mathbf{x}_t$.
  
  Compute the PQR expansion $\Tilde{\mathbf{x}}_t$ by (\ref{pqr_expansion})
  
  Let $I_t = \{i | \Tilde{\mathbf{x}}_{t,i} \neq 0\}$
  
  \For{all $i \in I_t$}{
     $\mathbf{w}_i = 
     \begin{cases}
        0 & \text{ if } |\mathbf{z}_i| < \lambda_1 \\
        -\frac{(\mathbf{z}_i - \mbox{sgn}(\mathbf{z}_i)\lambda_1)}{(\frac{\beta + \sqrt{\mathbf{n}_i}}{\alpha} + \lambda_2)} & \text{ if } |\mathbf{z}_i| \geq \lambda_1 \\
     \end{cases}$
  }
  compute 
  $p_t = \begin{cases}
    \mathbf{w}\cdot\Tilde{\mathbf{x}}_t & \text{ for rating prediction;} \\
    \mbox{sigmoid}(\mathbf{w}\cdot\Tilde{\mathbf{x}}_t) & \text{ for binary classification}.
  \end{cases}$

  \For{all $i\in I_t$}{
    $\mathbf{g}_i = (p_t - y_t)\mathbf{x}_{t, i}$
    
    $\sigma_i = \frac{1}{\alpha} (\sqrt{n_i + \mathbf{g}_i^2} - \sqrt{\mathbf{n}_i})$
    
    $\mathbf{z}_i \leftarrow \mathbf{z}_i + \mathbf{g}_i - \sigma_i \mathbf{w}_i$
    
    $\mathbf{n}_i \leftarrow \mathbf{n}_i + \mathbf{g}_i^2$
  }
 }
 \caption{FTRL-Proximal for the PQR model}\label{ftrl_pqr_alg}
\end{algorithm}

\paragraph{Implementation Details} We first calculate the separation $\mathcal{I}=\mathcal{H}\cup\mathcal{L}$. We can achieve this by sampling the data, calculating the feature counts and selecting top-$k$ features as $\mathcal{H}$ and the rest as $\mathcal{L}$. Then we find that PQR model is actually a set of rules to cross features. If $\mathbf{x}$ is a feature vector, we can expand it into a new vector, denoted by $\Tilde{\mathbf{x}}$. It can be defined by 
\begin{align}\label{pqr_expansion}
    \nonumber\Tilde{\mathbf{x}}_i &= \mathbf{x}_i \mbox{ if } i \in \mathcal{I} \\
    \nonumber\Tilde{\mathbf{x}}_{d + k i + j} &= \mathbf{x}_i\mathbf{x}_j \mbox{ if } i, j \in \mathcal{H} \\
             \Tilde{\mathbf{x}}_{d + k^2 + i} &= \mathbf{x}_i\mathbf{x}_j \mbox{ if } i \in \mathcal{H}, j \in \mathcal{L}
\end{align}
Actually we have the following definition.
\begin{defn}\label{pqr_expansion_def}
Given a feature vector $\mathbf{x}$ and a separation $\mathcal{I}=\mathcal{H}\cup\mathcal{L}$, the vector $\Tilde{\mathbf{x}}$ defined in (\ref{pqr_expansion}) is called the PQR expansion of $\mathbf{x}$.
\end{defn}
Therefore, the PQR model on a dataset is essentially equivalent to the linear model with respect the associated PQR expansions of this dataset. We list the implementation detail of online rating and online binary classification in Algorithm \ref{ftrl_pqr_alg} by following the FTRL-Proximal algorithm in \cite{McMahanKDD}.

\paragraph{More Explanation on PQR}
Like vanilla FM, the idea behind the PQR model is matrix factorization. In fact, the PQR matrix with respect to $\mathcal{I} = \mathcal{H}\cup\mathcal{L}$, can be considered as a decomposition in the following form
\begin{align}
    A^{(\mathcal{H}, \mathcal{L})} = P^T M P
\end{align}
where $P$ is $(k+1) \times d$ matrix and $M$ is a $(k+1)\times(k+1)$ matrix. $M$ is symmetric and its diagonals are zeros. Namely, $M_{i,j}=M_{j,i}$ and $M_{i,i} = 0$. So it has $\frac{1}{2}k(k+1)$ parameters which represents the feature-interaction learning parameters. The matrix $P$ is a projection matrix which maps the original feature vector into a lower dimensional space. Specifically, $Px$ is a permutation of
$$[x_{i_1}, x_{i_2}, \ldots, x_{i_k}, x_{\mathcal{L}}]$$
where $i_l\in \mathcal{H}$ with $l = 1, 2, \ldots, k$ and $x_{\mathcal{L}} = \sum_{j\in\mathcal{L}}x_j$. The key idea is that the projection matrix $P$ is determined by the separation $\mathcal{I}=\mathcal{H}\cup\mathcal{L}$, so we don't need to learn it. That is also the reason why we name the proposed model as projective quadratic regression model.

\section{Related Work}\label{related_sec}

\paragraph{Vanilla FM}
The vanilla Factorization Machine \cite{RendleICDM} defines the interaction matrix $A = V V^T$ where $V$ is a $d\times m$ matrix with $m \ll d$ and optimized on thin matrix $V$. It considers both the first and the second order information of features and learns the feature interactions automatically. It yields state-of-the-art performance in various learning tasks and its space complexity and computation cost are acceptable. However, it is a non-convex formulation and cannot fit into the online convex optimization setting. Moreover, the global optimal solution is not guaranteed by gradient based optimization algorithms. The bad local minima and saddle points may affect the performance and it is difficult to study how to select a good initial point theoretically.

\paragraph{CFM}
The Convex FM (CFM) introduced by Yamada \textit{et al.} \cite{YamadaKDD} is a convex variant of the widely used Factorization Machines. It employs a linear + quadratic model and regularizes the linear term with the $l_2$-norm and the quadratic term with the nuclear norm. It outperforms the vanilla FM on some applications. However, it is impractical in high-dimensional datasets due to high computation cost.

\paragraph{CCFM}
The Compact Convexified FM (CCFM) model invented by X. Lin \textit{et al.} \cite{LinWWW} is another convex model. By rewriting the bias term $b$, the first-order feature weight vector $\mathbf{w}$ and the second-order feature interaction matrix $A$ into a single matrix $C$ and restricting the feasible set by intersecting a nuclear ball, this formulation fits OCO setting. However, the computation cost is very high since at each iteration, a singular value decomposition operation is needed, which is unbearable in a real application.

\paragraph{DiFacto}
The Difacto introduced by M. Li \textit{et al.} \cite{LiDifacto} is a variant of Factorization Machines. It is not a convex model but it also use the information of frequency of feature occurrence. In DiFacto, the features are allocated embedding vectors with different ranks based on the occurrence frequency. Namely, higher frequency features corresponds to embedding vectors with larger ranks. The pairwise interaction between features with different ranks is obtained by simply truncating the embedding with high rank to a lower rank. Such truncations usually lead to worse performance. In fact, DiFacto reduces the computational burden of FM by sacrificing the performance. 

\section{Experiments}\label{exp_sec}
In this section, we evaluate the performance of the proposed PQR model on two popular machine learning tasks: online rating prediction and online binary classification.

\subsection{Experimental Setup}
We compare the empirical performance of PQR with state-of-the-art variants of FM in the online learning setting. We construct the baselines by applying online learning approaches to the existing formulations of FM: vanilla FM \cite{RendleICDM}, CFM \cite{YamadaKDD}, and DiFacto \cite{LiDifacto}.  
We also compare with LR model (linear regression of online rating prediction and logistic regression for online binary classification) since it is still popular in the high-dimensional sparse datasets. We skip CFM for high-dimensional datasets due to its high computation cost. For optimization method, We apply the state-of-the-art FTRL-Proximal algorithm \cite{McMahanKDD} with $\ell_1+\ell_2$ regularization for LR and PQR. Specifically, the implementation of PQR-FTRL can be seen in Algorithm \ref{ftrl_pqr_alg}. For vanilla FM, we just applied the online gradient descent with $\ell_1$ regularization. To summarize, the compared models are:

The vanilla FM and DiFacto are non-convex models so that we need to randomly initialize the learning parameters to escape local minimum. And results are different for each trial run. Therefore, we repeat the same experiment for 20 times. Finally, the average and standard deviation of 20 scores (RMSE, AUC or LogLoss) are reported. On the other hand, LR, CFM and PQR are convex models. So we can initialize the learning parameters as zeros (like in Algorithm-\ref{ftrl_alg} and Algorithm-\ref{ftrl_pqr_alg}). As long as the input data (train/test) are the same, the output scores are the same and thus there is no need for repetition. Moreover, for PQR model, we also test different orders ($k$ values) for each dataset.

\begin{table}
  \caption{Statistics of datasets}
  \label{dataset_stats}
  \centering
  \begin{tabular}{llll}
    \toprule
    Datasets     & \#Feature & \#Instances    & Label \\
    \midrule
    ML100K & 100,000  & 100,000  & Numerical   \\
    ML1M & 100,000  & 1,000,209  & Numerical \\
    ML10M & 200,000 & 10,000,054 & Numerical \\
    \midrule
    Avazu & 1,000,000 & 40,528,967 & Binary \\
    Criteo & 1,000,000 & 51,882,752 & Binary \\
    DD2012 & 54,686,452 & 149,009,105 & Binary \\
    \bottomrule
  \end{tabular}
\end{table}

\paragraph{Datasets} For the online rating prediction task, we use the typical MovieLens datasets, including MovieLens-100K (ML100K), MovieLens-1M (ML1M) and MovieLens-10M (ML10M); For the online binary classification task, we select high-dimensional sparse datasets including Avazu, Criteo\footnote{\url{http://labs.criteo.com/2014/02/kaggle-display-advertising-challenge-dataset/}}, and KDD2012\footnote{\url{https://www.kaggle.com/c/kddcup2012-track1/data}}. These three high-dimensional datasets are preprocessed and can be downloaded from the LIBSVM website\footnote{\url{https://www.csie.ntu.edu.tw/~cjlin/libsvmtools/datasets/}}. The statistics of the datasets are summarized in Table \ref{dataset_stats}. All these datasets are randomly separated into train(80\%), validation(10\%) and test(10\%) sets. In our experiments, the training instances are fed one by one to the model sequentially.

\begin{table*}[]
  \caption{RMSE on MovieLens datasets}
  \label{rmse_mvl}
  \centering
  \begin{tabular}{llllllll}
    \toprule
    RMSE   & LR     & FM                  & CFM    & DiFacto             & PQR ($k$=500) & PQR ($k$=1000) & PQR ($k$=2000)\\
    \midrule
    ML100K & 1.0429 & 1.0316 $\pm$ 1.2e-3 & 1.0326 & 1.0312 $\pm$ 9.0e-4 & 1.0225    & 1.0215     & \textbf{1.0215} \\
    ML1M   & 1.0487 & 1.0434 $\pm$ 5.0e-4 & 1.0403 & 1.0411 $\pm$ 9.0e-4 & 1.0321    & 1.0250     & \textbf{1.0190} \\
    ML10M  & 0.9651 & 0.9605 $\pm$ 1.7e-3 & 0.9616 & 0.9572 $\pm$ 1.1e-3 & 0.9547    & 0.9541     & \textbf{0.9536} \\
    \bottomrule
  \end{tabular}
\end{table*}

\begin{table*}
  \caption{AUC, LogLoss, Training Time, and Model Size of high-dimensional datasets}
  \label{auc_sparse}
  \centering
  \begin{tabular}{lllllll}
    \toprule
    Avazu         & LR        & FM                  & DiFacto             & PQR ($k$=2000)    & PQR ($k$=4000)    & PQR ($k$=8000)    \\
    \midrule
    AUC           & 0.7562    & 0.7752 $\pm$ 1.8e-4 & 0.7743 $\pm$ 1.1e-4 & 0.7757        & 0.7777        & \textbf{0.7785}        \\
    LogLoss       & 0.3939    & 0.3840 $\pm$ 1.1e-4 & 0.3840 $\pm$ 1.2e-4 & 0.3830        & 0.3818        & \textbf{0.3812}        \\
    Training Time & 1$\times$ & 58.5$\times$      & 58.0$\times$          & 3.3$\times$  & 3.6$\times$  & 3.8$\times$  \\
    Model Size    & 16.37K    & 1.17M               &  0.29M              & 0.68M          &  1.32M        & 1.97M         \\
    \midrule
    Criteo        & LR        & FM                  & DiFacto             & PQR ($k$=200) & PQR ($k$=500)     & PQR ($k$=1000)  \\
    \midrule
    AUC           & 0.7151    & 0.7218 $\pm$ 1.6e-4 & 0.7216 $\pm$ 1.6e-4 & 0.7207        & 0.7220        & \textbf{0.7221}   \\
    LogLoss       & 0.5116    & 0.5068 $\pm$ 2.0e-4 & 0.5070 $\pm$ 8.0e-5 & 0.5076        & 0.5067        & \textbf{0.5066}   \\
    Training Time & 1$\times$ & 20.1$\times$        & 20.6$\times$       & 17.1$\times$ & 18.3$\times$ & 19.2$\times$  \\
    Model Size    & 1.09K     & 17K                 & 14.6K               & 18.30K        & 77.68K        & 98.61K  \\
    \midrule
    KDD2012       & LR        & FM                  & DiFacto             & PQR ($k$=20)      & PQR ($k$=50)      & PQR ($k$=100)     \\
    \midrule
    AUC           & 0.7944    & 0.7968 $\pm$ 1.9e-4 & 0.7955 $\pm$ 3.0e-4 & 0.8013        & 0.8015        & \textbf{0.8016}        \\
    LogLoss       & 0.1541    & 0.1535 $\pm$ 5.1e-5 & 0.1534 $\pm$ 6.5e-5 & 0.1524        & 0.1524        & \textbf{0.1523}        \\
    Training Time & 1$\times$ & 10.4$\times$        & 11.8$\times$        & 3.1$\times$  & 3.1$\times$ & 3.2$\times$   \\
    Model Size    & 9.11M     & 91.6M              & 72.2M                & 9.70M         & 9.71M         & 9.71M         \\
    \bottomrule
  \end{tabular}
\end{table*}

\paragraph{Evaluation Metrics} To evaluate the performances on both tasks properly, we select different metrics respectively: the Root Mean Square Error (RMSE) for the rating prediction tasks; and the AUC (Area Under Curve) and LogLoss for the binary classification tasks. For the later case, we also compare the computation time and model size (the number of nonzero parameters) to demonstrate the efficiency of compared algorithms.

\subsection{Online Rating Prediction}

We use Movie Lens datasets for this task. For ML100K and ML1M, we use gender, age, occupation, and zipcode as the user feature and the movie genre as the item feature. We perform the standard discretization, use one-hot representation for each feature, and then concatenate them together. For ML10M, since there is no demographic information of users, we just use the user ids as the user feature instead.

The rank of FM is selected from the set of $\{2, 4, 8, 16, 32, 64, 128, 256, 512\}$. The coefficient of regularization term and the learning rate are tuned in a range $[0.0001, 10]$. We list the RMSE of all compared algorithms in Table \ref{rmse_mvl}. From our observation, PQR achieves higher prediction accuracy than all the baselines, which illustrates the advantage of PQR.

\subsection{Online Binary Classification}

In many real applications, feature vectors are usually extremely high-dimensional hence sparse representation is used instead, i.e., only nonzero key-value pairs are stored. We demonstrate the performance of PQR as well as other approaches for this case using high-dimensional sparse datasets: Avazu, Criteo, and KDD2012. Due to the high-dimensionality, CFM is impractical. So we compare PQR with LR, vanilla FM and DiFacto in two aspects: accuracy (measured by AUC and LogLoss) and efficiency (measured by training time and model size).

The rank for vanilla FM and DiFacto is selected from the set of $\{4, 8, 16, 32, 64\}$. The coefficient of regularization term and the learning rate are selected in a range of $[0.0001, 10]$. We list the results in Table \ref{auc_sparse}. In most cases, PQR outperforms the baseline algorithms in both accuracy and efficiency. The accuracy demonstrates that PQR has better expressiveness than linear model and low-rank FM models. The efficiency illustrates the theoretical computation cost of PQR model. Finally, for FM, there is no theory about how to choose a proper rank. However, the order of PQR has a clear meaning which is good for parameter tuning.

\section{Conclusion and Future Work}\label{con_sec}

In this paper, we propose a projective quadratic regression (PQR) model under the online learning settings. It meets the requirements in online convex optimization framework and the global optimal solution can be achieved due to its convexity. In addition, we show that the computation cost of PQR can be low if we choose a suitable order. Finally, we demonstrate its effectiveness by comparing with state-of-the-art approaches. For future work, a more scientific approach to select the order for PQR is an attractive direction.

\bibliographystyle{aaai}
\bibliography{AAAI-MaW.2016.bib}

\end{document}